\documentclass[lettersize,journal]{IEEEtran}
\usepackage{amsmath,amsfonts}
\usepackage{algorithmic}
\usepackage{algorithm}
\usepackage{array}
\usepackage[caption=false,font=normalsize,labelfont=sf,textfont=sf]{subfig}
\usepackage{textcomp}
\usepackage{stfloats}
\usepackage{url}
\usepackage{verbatim}
\usepackage{graphicx}
\usepackage{cite}
\usepackage{bm}
\usepackage{booktabs}
\usepackage{multicol}
\usepackage{multirow}
\usepackage{makecell}
\usepackage{color}
\newtheorem{theorem}{Theorem}

\begin{document}

\title{Learning-Accelerated Narrow-Phase Collision Detection via Check Ordering for \\ Sampling-Based Motion Planning}
\author{
	Hao Jiang, Yinghan Wang, Jianping He, Xiaoming Duan
	\thanks{The authors are with the Department of Automation, Shanghai Jiao Tong University, and Key Laboratory of System Control and Information Processing, Ministry of Education of China, Shanghai, China. Email: \small{\{mouse826612011, wyhboos, jphe, xduan\}@sjtu.edu.cn}.} }

\markboth{Journal of \LaTeX\ Class Files,~Vol.~14, No.~8, August~2021}%
{Shell \MakeLowercase{\textit{et al.}}: A Sample Article Using IEEEtran.cls for IEEE Journals}

\IEEEpubid{0000--0000/00\$00.00~\copyright~2021 IEEE}

\maketitle

\begin{abstract}
Collision detection is critical for ensuring the safety of planned paths in sampling-based motion planning.
However, it imposes a non-negligible computational burden on motion planners, motivating extensive studies on collision-detection acceleration. 
In commonly used phase-based collision-detection methods, the broad phase employs hierarchical structures to rapidly discard object pairs that are clearly collision-free, while the subsequent narrow phase performs detailed collision checks on the remaining object pairs whose collision status cannot be determined by  the broad phase.
Although these methods effectively reduce the number of detailed checks through broad-phase pruning, the narrow phase is usually executed in the default order returned by the broad phase, with little explicit optimization of the check order. 
This leaves room for further acceleration, especially in cluttered environments where many object pairs may remain after the broad phase and the narrow phase can account for a significant portion of the total detection time. In this work, we propose a learning-based method to accelerate phase-based collision detection by optimizing the check order in the narrow phase.
We first formulate the expected time cost of the narrow phase and derive an optimal check-ordering criterion that minimizes this expectation.
Since the priors required by this criterion are difficult to obtain in advance, we design a hypernetwork-based model to predict collision probabilities, which are then used to approximate the optimal check order. The resulting order guides the execution of exact mesh checks in the narrow phase, thereby reducing detection time without replacing the underlying geometric collision checker. Simulation results show that our method effectively accelerates phase-based collision detection and improves the efficiency and success rate of sampling-based motion planning, especially in cluttered environments.
\end{abstract}


\section{Introduction}
\label{sec:intro}
Sampling-based motion planning constructs feasible paths by sampling the robot's configuration space and connecting collision-free configurations~\cite{2014_elbanhawi_survey_motion_planning}.
Among its key processes, collision detection determines whether the robot collides with surrounding obstacles at each sampled configuration, serving as a basis for ensuring the safety of the generated paths. 
In typical geometric collision detection, the surface geometry of each object is represented as a mesh composed of numerous triangles. This applies to both robot components and environmental obstacles, where each triangle defines a small region of the object's surface.
Therefore, a direct collision check between two objects, referred to as a mesh check, involves a large number of  triangle-triangle intersection computations~\cite{jimenez20013d}, making it computationally expensive.
Since each sampled configuration may require checking many object pairs, collision detection often accounts for a substantial portion of the total computational cost in sampling-based motion planning~\cite{2012_RCIM_Duguleana_Obstacle_avoidance}.
Consequently, numerous studies have been conducted to accelerate collision detection~\cite{wu2024gpu}, and several studies have shown that such acceleration can improve the performance of sampling-based motion planners~\cite{2016_ICRA_Huh_mixture_models_for_fast_collision_detection}~\cite{2022_CoRL_Lai_parallelised_diffeomorphic}.

Among these studies, phase-based collision detection is one of the most widely used acceleration paradigms, including methods based on Bounding Volume Hierarchies (BVH)~\cite{klosowski1998efficient}, Octrees~\cite{jung1997octree}, and Spatial Hashing~\cite{2003_teschner_spatial_hashing}.
These methods divide the detection into two phases: the broad phase and the narrow phase.
In the broad phase, coarse geometric representations, such as bounding volumes or spatial partitions, are used to perform efficient overlap tests and quickly discard object pairs that are clearly collision-free.  
Then, in the subsequent narrow phase, mesh checks are performed only on the remaining object pairs whose collision status cannot be determined by the broad phase. 
Therefore, phase-based methods achieve acceleration primarily through broad-phase pruning, and their differences mainly lie in the coarse representations and exclusion strategies used in the broad phase.
 
 \IEEEpubidadjcol

Compared with the broad phase, the narrow phase has received relatively limited attention in existing phase-based methods. In most implementations, mesh checks are performed on the remaining object pairs from the broad phase one by one, following the default order returned by the broad-phase traversal, without explicit check-order optimization. 
In cluttered environments, the robot is often in close proximity to obstacles, which limits the ability of the coarse broad phase to exclude object pairs. 
As a result, many object pairs still need to be processed during the narrow phase.
Due to the high computational complexity of the mesh check, the narrow phase can account for a significant portion of the overall detection time, as shown in Fig.~\ref{fig:research_content}. 
This indicates that optimizing the narrow phase can provide further acceleration for phase-based collision detection.

 \begin{figure*}[!t]
	\centering
	\includegraphics[width=6.0in]{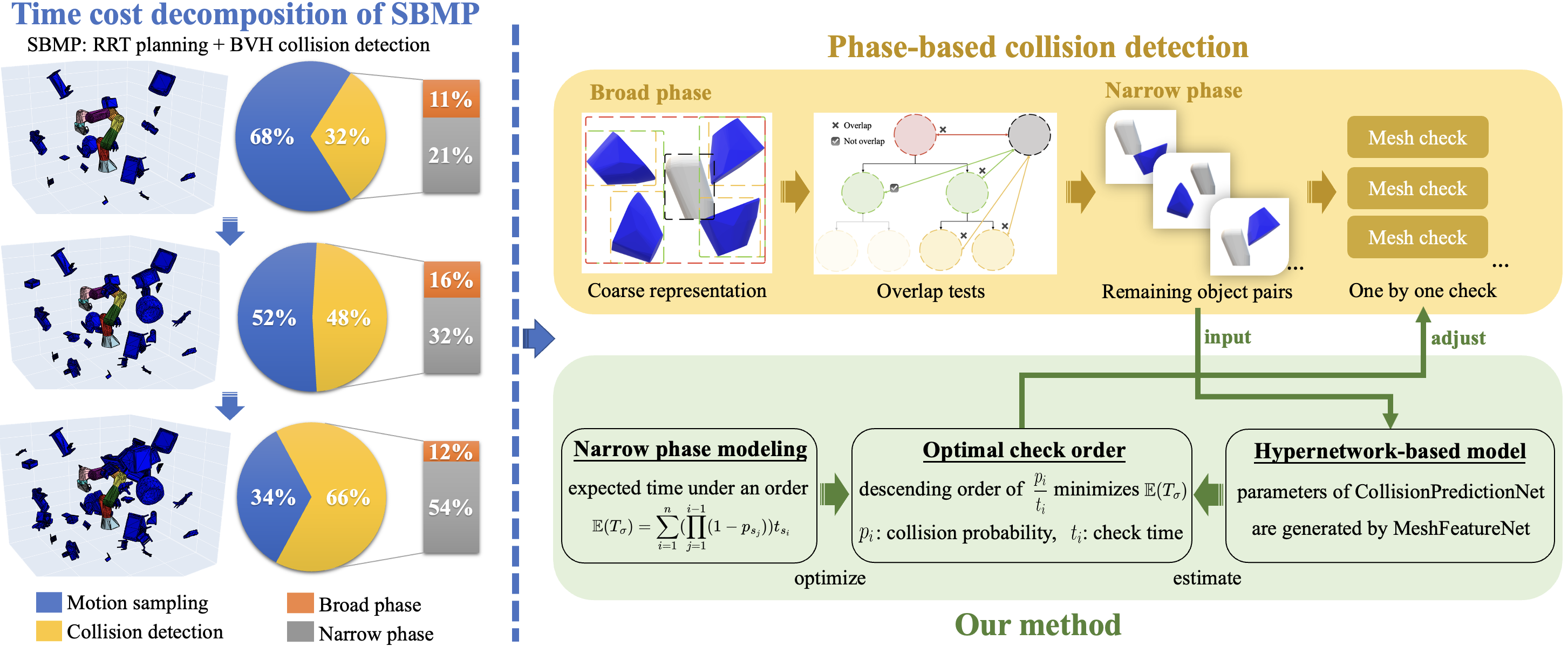}
	\caption{
		Overview of this research. 
		As the environment becomes increasingly cluttered, collision detection accounts for a growing proportion of the total computational cost in sampling-based motion planning, with the narrow phase becoming a major bottleneck.  
		To address this issue, we propose a learning-based method to accelerate the narrow phase by optimizing its check order.  Specifically, a hypernetwork-based model predicts collision probabilities for object pairs, which are used to approximate the optimal check order and enable earlier termination in collision cases, thereby improving the efficiency of the overall detection process.
	}
	\label{fig:research_content}
\end{figure*}

Therefore, in this work, we aim to accelerate the narrow phase by adjusting the order of mesh checks, so that object pairs more likely to collide are examined earlier. 
This enables earlier termination once a collision is detected, thereby reducing the expected detection time. 
However, constructing an effective check order requires prior information about the object pairs, such as their collision probabilities, which is difficult to obtain in advance. 
To address this, we propose a hypernetwork-based learning model to predict the required priors and use them to guide narrow-phase check ordering. Importantly, the proposed method does not replace exact mesh checks with learning-based predictions; instead, it only reorders exact checks, preserving the reliability of the underlying geometric collision detector. In summary, our main contributions are:
\begin{itemize}
    \item We propose to accelerate the phase-based collision detection by optimizing the check order in the narrow phase. By modeling the expected time cost of the narrow phase, we derive an optimal check-ordering criterion that minimizes this expectation. 
    \item We design a hypernetwork-based model to predict the priors required  for constructing an approximate optimal check order, with an emphasis on generalization and low computational complexity. 
    \item Simulation results show that the proposed method accelerates phase-based collision detection in cluttered environments, thereby improving both the efficiency and success rate of sampling-based motion planners.
\end{itemize}

\section{Background}
Collision detection determines whether the meshes of robot components intersect  with those of environmental obstacles. A straightforward implementation is to perform mesh checks on all object pairs. However, since mesh checks are computationally expensive, collision detection can become time-consuming, motivating extensive research on acceleration methods. 
Existing approaches can be broadly categorized into phase-based methods, which reduce the number of mesh checks, and learning-based methods, which reduce the computational cost of individual checks.

Phase-based methods divide collision detection into two phases: the broad phase and the narrow phase. Acceleration is primarily achieved in the broad phase, where coarse geometric representations are used to enclose objects and enable simple overlap tests. Object pairs that clearly do not collide can then be quickly discarded without performing detailed mesh checks. 
For example, BVH~\cite{klosowski1998efficient}  organizes objects or geometric primitives into a hierarchy of bounding volumes. Since these volumes are usually  simple shapes, such as boxes or spheres, their overlap can be tested much more efficiently than mesh intersections. If two bounding volumes do not overlap, the object pairs contained in them can be directly excluded from further mesh checks.
Similarly, Octrees~\cite{jung1997octree} recursively divide the environment into octants and construct a hierarchical spatial index, thereby reducing the number of object pairs requiring mesh checks by excluding non-intersecting nodes. 
Spatial Hashing~\cite{2003_teschner_spatial_hashing} partitions the workspace into grid  cells and checks only object pairs mapped to the same or neighboring cells, efficiently filtering out pairs that are guaranteed to be collision-free. 
In addition, some studies have explored parallel or hardware-friendly implementations of phase-based collision detection~\cite{ramsey2024collision}. 
Although these methods differ in their broad-phase representations and exclusion strategies, their subsequent narrow phase is usually similar: mesh checks are performed on the remaining object pairs in the order returned by the broad phase, without explicit optimization of the check order.

In contrast, learning-based methods aim to approximate mesh-check results using data-driven models, achieving faster predictions at the cost of possible prediction errors. 
Early studies used multilayer perceptrons~\cite{2002_ICSMC_Garc_box_collision_detection} and Support Vector Machines (SVMs)~\cite{pan2015efficient} to learn collision boundaries, but these methods were mainly limited to relatively simple objects and offered limited advantages in prediction efficiency.
To improve efficiency, Pan et al.~\cite{pan2016fast} proposed a method that stores prior collision results and uses them to predict collisions for new robot poses, based on the observation that nearby poses may share similar collision states.  Das et al.~\cite{2020_TRO_Das_fastron} learned a collision proxy using a perceptron classifier with a computationally efficient kernel, and further showed  that a forward-kinematics kernel is more suitable for measuring the similarity between two robot poses~\cite{2020_RAL_Das_fastron_fk}. 
Verghese et al.~\cite{2020_RAL_Verghese_fastron_decomposition} and Han et al.~\cite{2020_IROS_Han_configuration_space_decomposition} further decomposed the configuration space into multiple subspaces, each associated with a separate predictor, thereby improving prediction accuracy and reducing prediction time. 
These methods mainly focus on the trade-off between prediction accuracy and computational complexity. However, they are typically designed for fixed environments and may suffer from limited generalization. Moreover, when learning-based predictions are used to replace exact mesh checks, false predictions may waste computational resources or even compromise path safety. Therefore, deterministic phase-based methods remain widely used in motion planning systems. Although learning-based predictors could be applied to the narrow phase, directly replacing mesh checks still introduces the risks of incorrect predictions and limited generalization.

To further accelerate phase-based collision detection while preserving deterministic collision-checking results, this paper focuses on optimizing the check order in the narrow phase. Some heuristic strategies, such as the surface-area-based~\cite{goldsmith1987automatic} and 
distance-based~\cite{gunther2007realtime} rules, use geometric features as indirect proxies for collision likelihood and can be used to adjust the check order. However, these heuristics do not explicitly model the expected time cost of the narrow phase or provide an optimality criterion for check ordering. 
Therefore, we formulate the expected narrow-phase time cost and derive an optimal check-ordering criterion. Instead of replacing exact mesh checks, the proposed hypernetwork-based model predicts the priors required to approximate this criterion and uses them only to reorder mesh checks in the narrow phase.

\section{Narrow-Phase Optimization in \\ Phase-based Collision Detection}
\label{sec:definition}
In this section, we first introduce the principle of phase-based collision detection using BVH as an example, and discuss the potential for optimization in the narrow phase.
Then, we formulate a time cost model for the narrow phase and derive an optimal check-ordering criterion that minimizes the expected time cost. 
Next, we propose a hypernetwork-based model to learn and predict the priors required to approximate the optimal check order.
Finally, we summarize the proposed learning-accelerated phase-based collision detection method for sampling-based motion planning.

\subsection{Phase-based Collision Detection}
\label{sec:3_1}

\begin{figure*}[t!]
	\centering
	\subfloat[Robot model]{\includegraphics[width=1.5in]{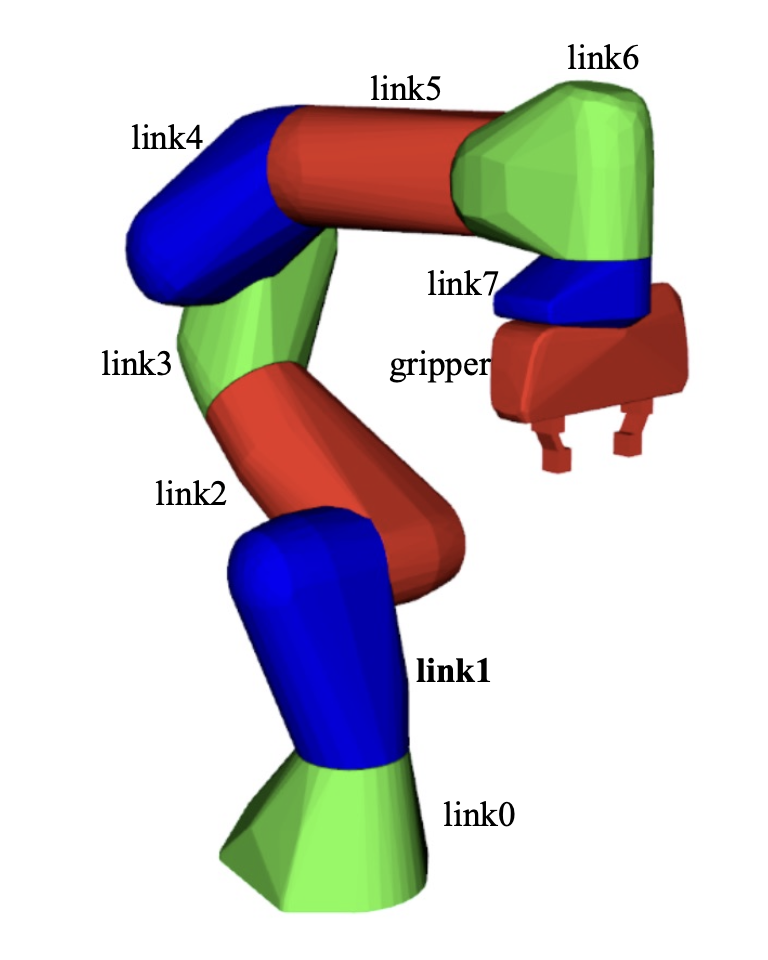}
		\label{fig:panda_arm}}
	\subfloat[Bounding boxes]{\includegraphics[width=2.3in]{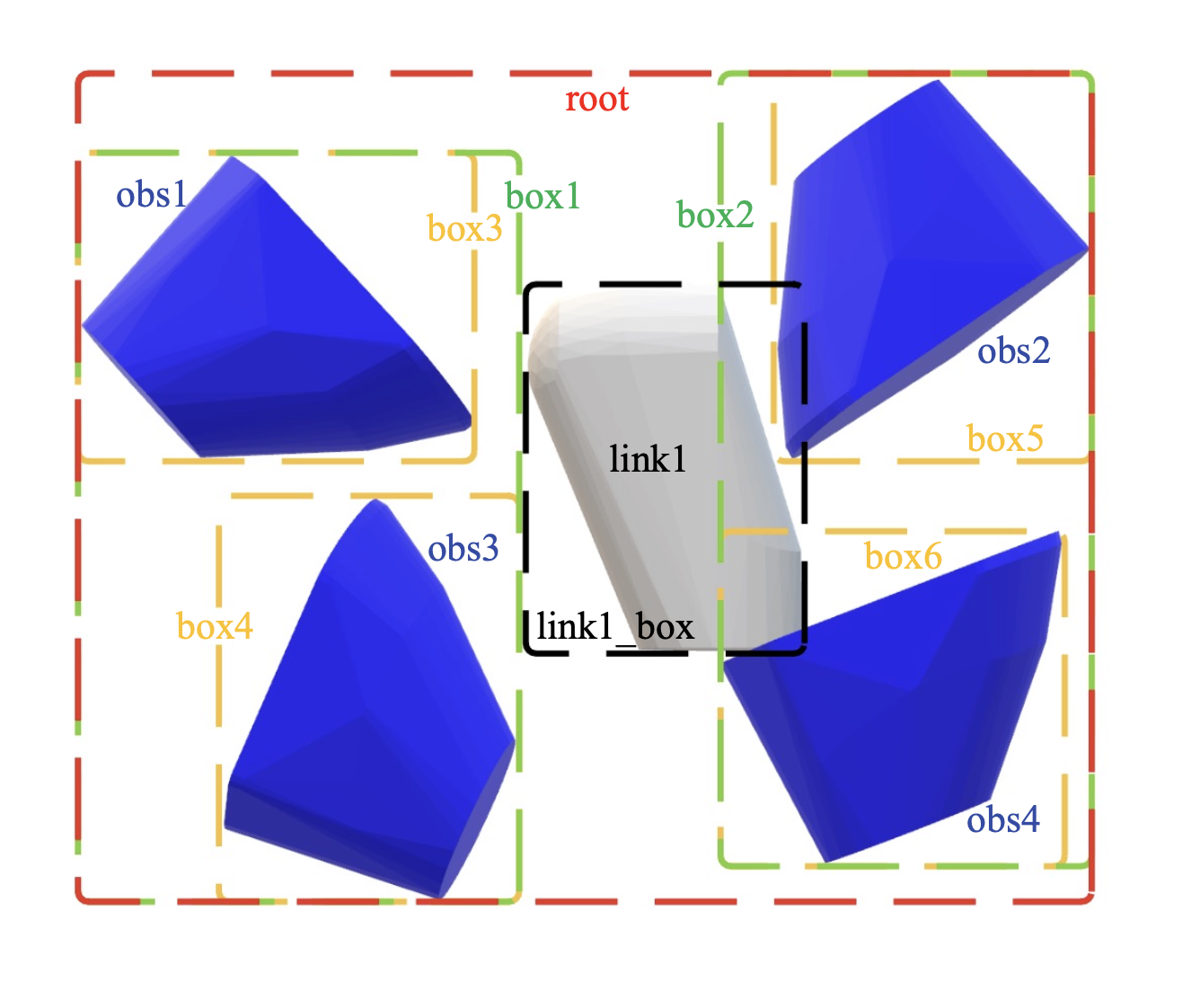}
		\label{fig:bounding}}
	\subfloat[Phases in BVH]{\includegraphics[width=3.0in]{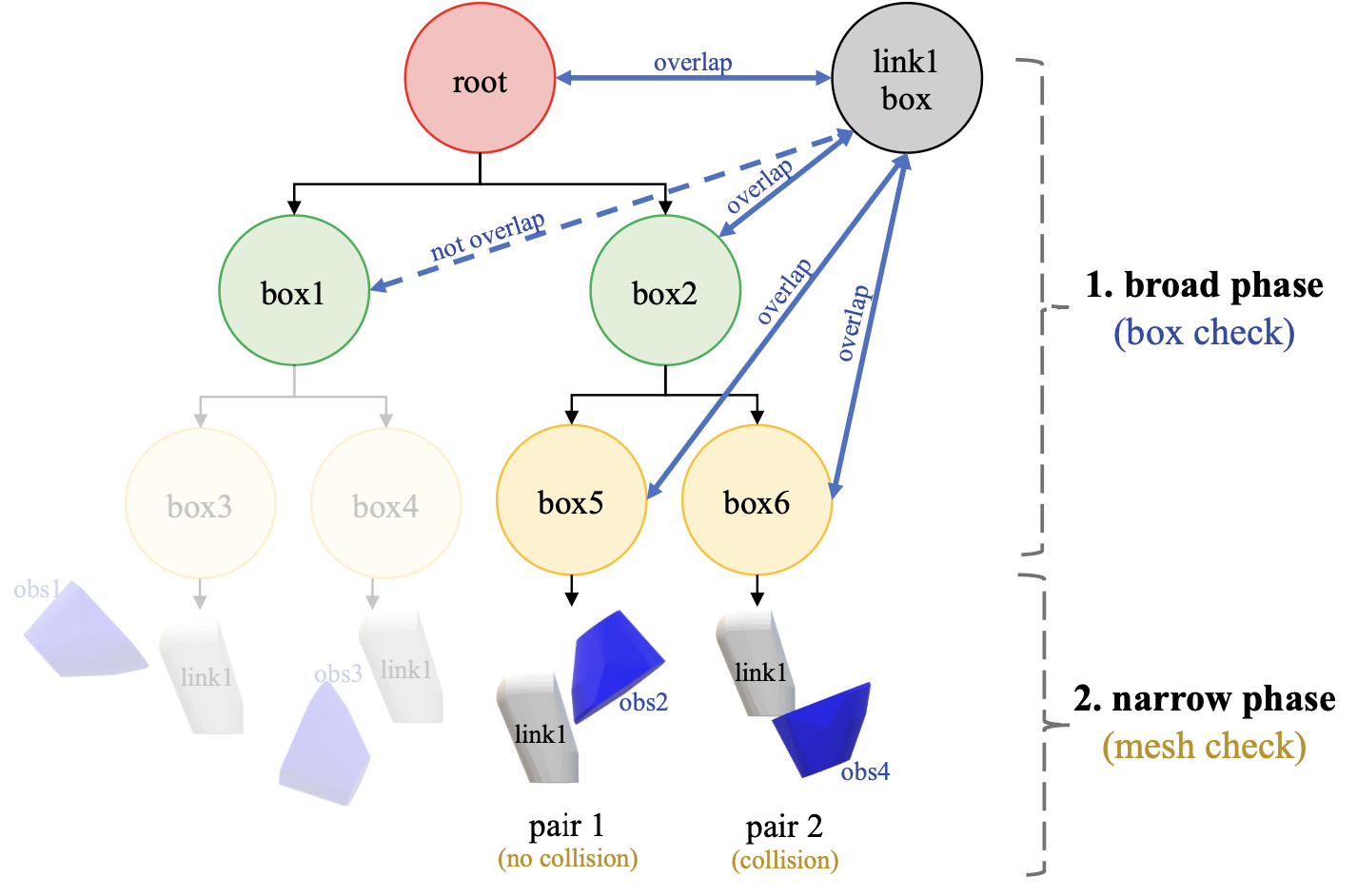}
		\label{fig:bvh}}
	\caption{Collision detection based on BVH.
		(a) The Franka Emika Panda robot model, which consists of nine components.
		(b) An environment with four obstacles, where AABBs are used as the bounding volumes to construct the hierarchy. A root box encloses all obstacles; box1 and box2 enclose the left and right obstacle groups,  respectively; and box3 to box6 enclose individual obstacles. 
		(c) The broad phase forms a BVH tree and performs box checks along the tree. If a parent node, such as box1, does not overlap with link1\_box, then its child nodes, box3 and box4, can be excluded. 
		In the subsequent narrow phase,  leaf nodes that still overlap with link1\_box, such as box5 and box6, are checked by mesh checks on the corresponding object pairs, pair1 and pair2, to determine whether an actual collision occurs.
	}
	\label{fig:bvh_principle}
\end{figure*}

Phase-based methods divide the collision detection process into the broad phase and the narrow phase.
The basic idea is to use bounding volumes and hierarchical structures in the broad phase to quickly exclude object pairs that are clearly collision-free, thereby reducing the number of remaining object pairs that require detailed mesh checks in the narrow phase. In this way, phase-based methods avoid performing mesh checks on all object pairs and thus accelerate collision detection. 
We use the widely adopted BVH as an example, as shown in Fig.~\ref{fig:bvh_principle}. Suppose  BVH is applied to detect collisions between the robot component link1 and four obstacles.

\begin{itemize}
	\item Broad phase: As illustrated in Fig.~\ref{fig:bounding}, bounding volumes are constructed for link1 and each obstacle, and a hierarchy, i.e., a BVH tree, is built for the bounding volumes of all obstacles.
	Since bounding volumes are usually simple shapes, such as the Axis-Aligned Bounding Boxes (AABBs) used in Fig.~\ref{fig:bounding}, testing whether two boxes overlap, referred to as a box check, is much simpler than performing a mesh check. Moreover, if two boxes do not overlap, the objects enclosed by them cannot collide. 
	Therefore, based on the BVH tree in Fig.~\ref{fig:bvh}, if a parent box does not overlap with the box of link1, its child boxes can be directly excluded. Otherwise, box checks are further performed for its child boxes until the leaf nodes of the tree are reached. 
	
	\item Narrow phase:  Since the bounding-volume hierarchy used in the broad phase is coarse, an overlap between a leaf node in the BVH tree and the box of link1 does not necessarily indicate an actual collision. As a result, in the narrow phase shown in Fig.~\ref{fig:bvh}, mesh checks are  still required for the object pairs  remaining from the broad phase, because their collision status cannot be determined by the box check alone. 
	The default check order in the narrow phase is determined by the order in which the remaining object pairs are generated during broad-phase traversal, and this order is not explicitly optimized for narrow-phase collision checking.
\end{itemize}

Existing phase-based methods mainly differ in the hierarchical structures, bounding volumes, and exclusion strategies used in the broad phase.  However, the remaining object pairs  after the broad phase are typically  checked one by one using mesh checks in  the narrow phase.
This leaves room for further acceleration, especially in cluttered environments. In such environments, the robot may be close to obstacles, limiting the ability of the broad phase to exclude object pairs and leaving many pairs to be processed in the narrow phase. 
Since mesh checks are much more expensive than box checks, the narrow phase can account for a large portion of the overall detection time.
Therefore, phase-based collision detection can be further improved by optimizing the narrow phase.

\subsection{Narrow Phase Modeling and Optimal Check Order}
\label{sec:3_2}
In the narrow phase of phase-based collision detection, mesh checks are  performed on object pairs whose collision status cannot be determined by  the broad phase.
Since collision detection only needs to determine whether at least one collision occurs, the narrow phase terminates either after all object pairs are verified as collision-free or immediately once a collision is detected.
Therefore, to reduce the expected time cost of the narrow phase, object pairs with higher collision probabilities and lower check costs should be examined earlier. 
This enables earlier termination of the narrow phase. 
Building on this idea, we formulate the time cost of the narrow phase and derive an optimal check-ordering criterion.

Suppose there are $n \geq 2$ object pairs $\{m_1, m_2, \dots, m_n\}$ that require  mesh checks in the narrow phase. For each object pair $m_i$, let $p_i$ denote its collision probability and $t_i$ denote the time cost of its mesh check. A check order is denoted by a permutation $\bm{\sigma} = [s_1, s_2, \dots, s_n]$, where $s_i$ is the index of the object pair checked at the $i$-th position. 
Assuming that collision events of different object pairs are independent, the expected time cost $\mathbb{E}(T_{\bm{\sigma}})$ of the narrow phase under the order $\bm{\sigma}$ is
 \begin{align}\label{eq:expectation}
 	\begin{split}
 		\mathbb{E}(T_{\bm{\sigma}})=&t_{s_1}+(1-p_{s_1})t_{s_2} + (1-p_{s_1})(1-p_{s_2})t_{s_3} \\
 		&+\cdots+(1-p_{s_1})\cdots(1-p_{s_{n-1}})t_{s_n}\\
 		=&\sum_{i=1}^{n} (\prod_{j=1}^{i-1} (1-p_{s_j}))t_{s_i}.    
 	\end{split}
 \end{align}
In the following, we show that the expected time cost is minimized by sorting object pairs in nonincreasing order of $\frac{p_i}{t_i}$.
 
 \begin{theorem}
 	\label{theorem1}
    Assume that collision events of different object pairs are independent. An optimal check order that minimizes the expected time cost $\mathbb{E}(T_{\bm{\sigma}})$ in~\eqref{eq:expectation} is given by
	$\bm{\sigma}^o=[s^o_1,s^o_2,\dots,s^o_n]$, where for any $i<j$, $\frac{p_{s^o_i}}{t_{s^o_i}}\geq\frac{p_{s^o_j}}{t_{s^o_j}}$.
 \end{theorem}

\textit{Proof}: 
Consider an arbitrary check order  $\bm{\sigma}^a=[s_1,\dots,s_l,s_{l+1}\dots,s_n]$ and another order $\bm{\sigma}^b=[s_1,\dots,s_{l+1},s_{l}\dots,s_n]$ obtained by swapping the adjacent elements $s_{l}$ and $s_{l+1}$ of $\bm{\sigma}^a$, where $l\in\{1,2,\dots,n-1\}$. 

According to \eqref{eq:expectation}, the expected time costs of the narrow phase under $\bm{\sigma}^a$ and $\bm{\sigma}^b$ are
\begin{equation}
	\label{eq:proof_3}
	\begin{aligned}
		\mathbb{E}(T_{\bm{\sigma}^a})&=\sum_{i=1}^{l-1} (\prod_{j=1}^{i-1} (1-p_{s_j}))t_{s_i}  \\
		&\quad+\prod_{j=1}^{l-1} (1-p_{s_j})(t_{s_l} + (1-p_{s_l})t_{s_{l+1}})\\
		&\quad +\sum_{i=l+2}^{n} (\prod_{j=1}^{i-1} (1-p_{s_j}))t_{s_i},
	\end{aligned} 
\end{equation}
and
\begin{equation}
	\label{eq:proof_4}
	\begin{aligned}
		\mathbb{E}(T_{\bm{\sigma}^b})&=\sum_{i=1}^{l-1} (\prod_{j=1}^{i-1} (1-p_{s_j}))t_{s_i}  \\
		&\quad +\prod_{j=1}^{l-1} (1-p_{s_j})(t_{s_{l+1}} + (1-p_{s_{l+1}})t_{s_{l}})\\
		&\quad +\sum_{i=l+2}^{n} (\prod_{j=1}^{i-1} (1-p_{s_j}))t_{s_i}.\\
	\end{aligned}
\end{equation}
Subtracting~\eqref{eq:proof_4} from~\eqref{eq:proof_3}, we obtain the time difference as
\begin{align*}
        &\quad \mathbb{E}(T_{\bm{\sigma}^a}) -\mathbb{E}(T_{\bm{\sigma}^b})\\&=\prod_{j=1}^{l-1} (1-p_{s_j})(t_{s_{l}} +  
		(1-p_{s_{l}})t_{s_{l+1}}-t_{s_{l+1}} - (1-p_{s_{l+1}})t_{s_{l}})\\
        &=\prod_{j=1}^{l-1} (1-p_{s_j})(p_{s_{l+1}}t_{s_{l}}-p_{s_{l}}t_{s_{l+1}}).
\end{align*}
Thus, it follows that $\mathbb{E}(T_{\bm{\sigma}^a}) > \mathbb{E}(T_{\bm{\sigma}^b})$ if and only if $\frac{p_{s_l}}{t_{s_l}}<\frac{p_{s_{l+1}}}{t_{s_{l+1}}}$. In other words, if $\frac{p_{s_l}}{t_{s_l}}<\frac{p_{s_{l+1}}}{t_{s_{l+1}}}$, then swapping $s_l$ and $s_{l+1}$ results in a  check order with a lower expected time cost in the narrow phase.
Moreover, for any check order whose elements are not in nonincreasing order of $\frac{p_i}{t_i}$, we can swap neighboring elements to construct a better check order that reduces the expected time cost of the narrow phase, until the optimal check order $\bm{\sigma}^o$ is achieved. \hfill{\IEEEQEDopen}

\subsection{Hypernetwork-based model}

\begin{figure}[!b]
	\centering
	\includegraphics[width=2.8in]{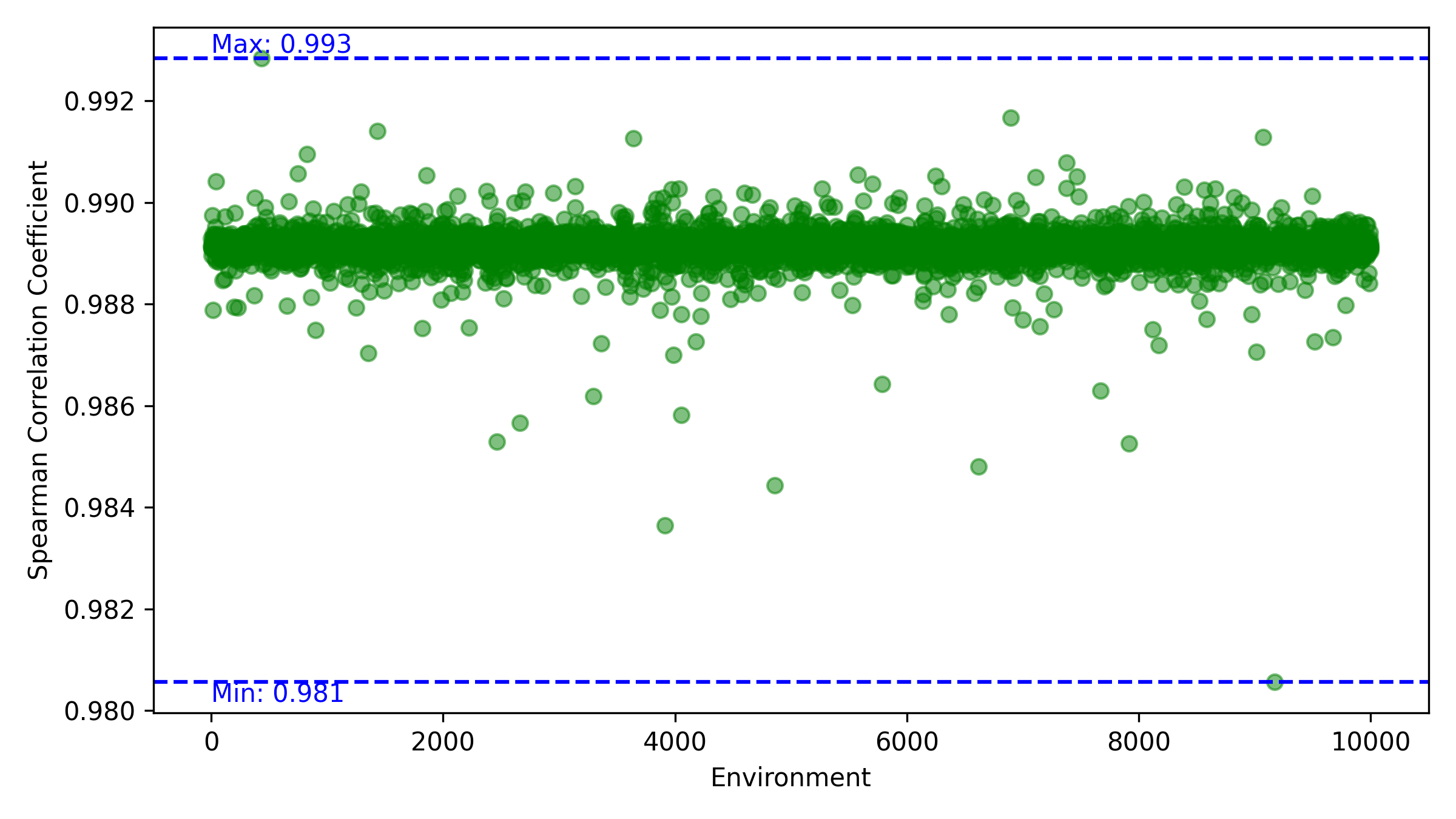}
	\caption{Spearman rank correlation~\cite{spearman1987proof} between the check order obtained by sorting object pairs by $\hat p_i$ in descending order and the optimal check order $\bm{\sigma}^o$. The correlation remains above $0.98$ across $10000$ randomly generated narrow-phase cases on Apple M3.
	}
	\label{fig:data_driven_proof}
\end{figure}

\begin{table}[!b]
	\scriptsize
	\centering
	\caption{
		\textbf{[Validation of Approximating $\hat{p}/t$ with $\hat{p}$] Average variable correlation and order similarity across different CPUs.}
	}
	\label{tab:gpus_proxy}
	\centering
	\begin{tabular}{ccccc}
		\toprule
		\multirow{2}{*}{\textbf{CPU}} & \multicolumn{2}{c}{\makecell[c]{\textbf{Variable correlation}\\ (Pearson coefficient~\cite{benesty2009pearson})}} & \multicolumn{2}{c}{\makecell[c]{\textbf{Order similarity}\\(Spearman rank correlation~\cite{spearman1987proof} )}} \\
		
		& $1/t$ vs. $\hat{p}/t$ & $\hat{p}$ vs. $\hat{p}/t$ & $1/t$ vs. $\hat{p}/t$& $\hat{p}$ vs. $\hat{p}/t$ \\
		\midrule
		Intel i7-11800H & $0.265$ & $0.883$ & $0.605$ & $0.998$\\
		\midrule
		AMD 9950X3D & $0.171$ & $0.799$ & $0.495$ & $0.994$\\
		\midrule
		Apple M3 & $0.323$ & $0.841$ & $0.469$ & $0.997$\\
		\bottomrule
	\end{tabular}
\end{table}

Although the optimal check-ordering criterion is derived in Theorem~\ref{theorem1}, the priors $(p_i,t_i)$ required to construct the optimal order are difficult to obtain in advance. 
We address this issue using a learning-based method, with the following considerations. First, since prediction and check-order construction introduce additional operations into phase-based collision detection, the learning model should maintain a balance between prediction accuracy and computational complexity. Second, the model should generalize well to environments and object meshes beyond those used during training. 
Third, since the check time $t_i$ is difficult to predict accurately and may vary across hardware platforms, we use the predicted collision probability $\hat p_i$ as a practical proxy for the optimal score $p_i/t_i$. This approximation is motivated by the empirical observation in Fig.~\ref{fig:data_driven_proof}, where the check order obtained by sorting object pairs in descending order of $\hat p_i$ is highly similar to the optimal order $\bm{\sigma}^o$ computed using $\hat{p}_i/t_i$, where $\hat p$ is predicted by our model. 
To further validate this approximation,  we collect statistics from extensive computations on different hardware platforms, involving $10000$ narrow phase cases, $100000$ object pairs, and more than $2000$ meshes of varying shapes. The results are summarized in Table~\ref{tab:gpus_proxy}. 
On the one hand, using $\hat p$ as a proxy for $\hat{p}/t$ produces a check order with high correlation and order similarity to that obtained from $\hat{p}/t$. On the other hand, although the distribution of $t_i$ varies across hardware platforms, its variation is  relatively small and does not  substantially change the ordering induced by $\hat{p}/t$. Since our study focuses on the ordering rather than the exact values of $\hat{p}/t$, using $\hat p$ as a practical proxy is reasonable. 

\begin{figure*}[t!]
	\centering
	\includegraphics[scale=0.210]{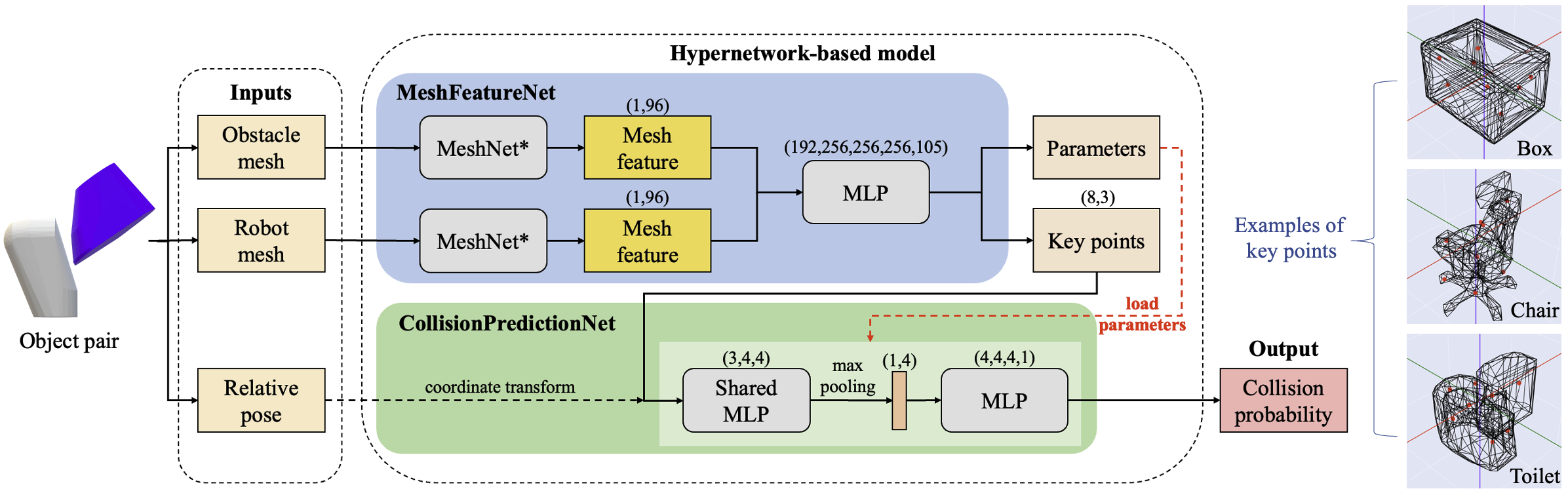}
	\caption{
		Structure of the hypernetwork-based model.
		MeshNet* is a lightweight variant of MeshNet~\cite{2019_AAAI_Feng_Meshnet}, where only one MeshConv layer is retained and the dimensions of its three outputs are reduced to 32 for lower computational complexity. This simplification is feasible because MeshNet* is used for compact mesh-feature extraction rather than the original shape-classification task.
	}
	\label{fig:NetStructure}
\end{figure*}

Then, we propose a hypernetwork-based model to predict the collision probability of each object pair in the narrow phase, as shown in  Fig.~\ref{fig:NetStructure}.
The model consists of a MeshFeatureNet~(MFN) and a CollisionPredictionNet~(CPN), where the network parameters of CPN are generated by MFN.
The rationale behind this design is that  object pairs differ in two aspects: the meshes of the two objects and their relative pose. 
In sampling-based motion planning, the meshes of the robot and obstacles usually remain fixed, while their relative poses change frequently.
As a result, during collision-probability prediction, the component  related to relative poses is invoked much more frequently than the component  related to meshes.
Moreover, since meshes are much more complex than relative poses, using a single shared network to process both mesh geometry and relative pose makes it difficult to balance computational complexity and generalization. 
Most existing collision-prediction methods rely on such shared structure, which often prioritizes low complexity at the cost of generalization~\cite{pan2016fast}~\cite{2020_RAL_Das_fastron_fk}. 
Inspired by hypernetworks~\cite{chauhan2024brief}, which generate the parameters of a target network using another network, we use MFN to generate the parameters of CPN.
In this way, mesh-related information and pose-related information are handled separately, enabling a better balance between efficiency and generalization. Specifically, the two components are designed as follows: 
\begin{itemize}
	\item MeshFeatureNet: MFN uses two MeshNets~\cite{2019_AAAI_Feng_Meshnet} and one Multilayer Perceptron (MLP) to extract features from two input meshes. It outputs the parameters of CPN and a set of key points that compactly encode the geometry of the object pair. 
	In this way, MFN focuses on the mesh-related information and transfers this information to CPN through its outputs. 
	After training, the outputs of MFN are stored in memory, and a library is maintained for object pairs with known meshes. It is worth noting that, although MFN has a more complex architecture than CPN, its forward pass is only required when an object pair with previously unknown meshes appears during prediction.
	
	\item CollisionPredictionNet: CPN loads the parameters generated by  MFN and predicts the collision probability. Although its external input is only the relative pose, the loaded parameters and the key points transformed by the relative pose implicitly encode mesh-related information. Therefore, CPN only needs to model the collision probability under different relative poses for two specific meshes, making it possible to use a compact network structure for efficient prediction.
\end{itemize}
When the environment remains unchanged, the proposed model can directly use the CPN with loaded parameters to make predictions.
When unknown obstacles appear, object pairs with new meshes only need to pass through MFN once to obtain the corresponding CPN parameters and key points. Afterward, predictions for new poses are still made by the loaded CPN.
Therefore, although MFN is more computationally expensive than CPN, its impact on the overall detection time is negligible due to its infrequent use.
As a result, the proposed MFN-CPN structure maintains low computational  complexity while improving generalization.

The hypernetwork-based model is trained in a supervised manner using labeled collision-check samples of the form $\{(m_{\mathrm{rob}_i}, m_{\mathrm{obs}_j}, T_{i,j,k}, c_{i,j,k})\}_{k = 1,\dots,N_{i,j}}$ for each object pair $(i,j)$, where $N_{i,j}$ is the number of relative poses, $m_{\mathrm{rob}_i}$ denotes the  mesh of the $i$-th robot component,  $m_{\mathrm{obs}_j}$ denotes  the mesh of the $j$-th obstacle, $T_{i,j,k}$ is the $k$-th relative pose between the two meshes, and $ c_{i,j,k}\in \{0,1\}$ is the corresponding mesh-check label.
During training, $(m_{\mathrm{rob}_i}, m_{\mathrm{obs}_j}, T_{i,j,k})$ are used as inputs, and $c_{i,j,k}$ is used as the supervised signal. Since the parameters of CPN are generated by MFN, the trainable parameters are those of MFN. The forward propagation of the model is
\begin{equation}
\left\{
\begin{aligned}
	& (\varphi_{i,j}, P_{i,j})  \gets \mathrm{MFN}(m_{\mathrm{rob}_i}, m_{\mathrm{obs}_j} ; \theta), \\
	& P'_{i,j,k} = R_{i,j,k} \cdot P_{i,j} + v_{i,j,k}, \\
	& \hat{p}_{i,j,k} \gets \mathrm{CPN}(P'_{i,j,k}; \varphi_{{i,j}} ),
\end{aligned}
\right.
\end{equation}
where $\theta$ denotes the parameters of MFN, $\varphi_{i,j}$ denotes the generated parameters of CPN, and $P_{i,j}$ denotes the key points generated by the MFN, $P'_{i,j,k}$ denotes the key points transformed according to the relative pose  $T_{i,j,k}$, and $R_{i,j,k}$ and $v_{i,j,k}$ are the rotation matrix and translation vector induced by $T_{i,j,k}$. 
The model is trained using the cross-entropy loss 
\begin{equation}
\label{eq:loss}
		l(\theta) = -\sum_{i,j} \sum_{k=1}^{N_{i,j}}  (c_{i,j,k} \ln \hat{p}_{i,j,k} + (1- c_{i,j,k})\ln(1- \hat{p}_{i,j,k})).
\end{equation}

In summary, the proposed acceleration method for phase-based collision detection in sampling-based motion planning is illustrated in Fig.~\ref{fig:our_method}. 
The method focuses on the narrow phase and accelerates collision detection by adjusting the check order of mesh checks.
First, the hypernetwork-based model predicts the collision probability of each object pair remaining after the broad phase. Then, these object pairs are sorted in descending order of their predicted collision probabilities to approximate the optimal check order.  
Finally, exact mesh checks are performed according to the constructed check order.

\begin{figure}[!h]
	\centering
	\includegraphics[width=3.45in]{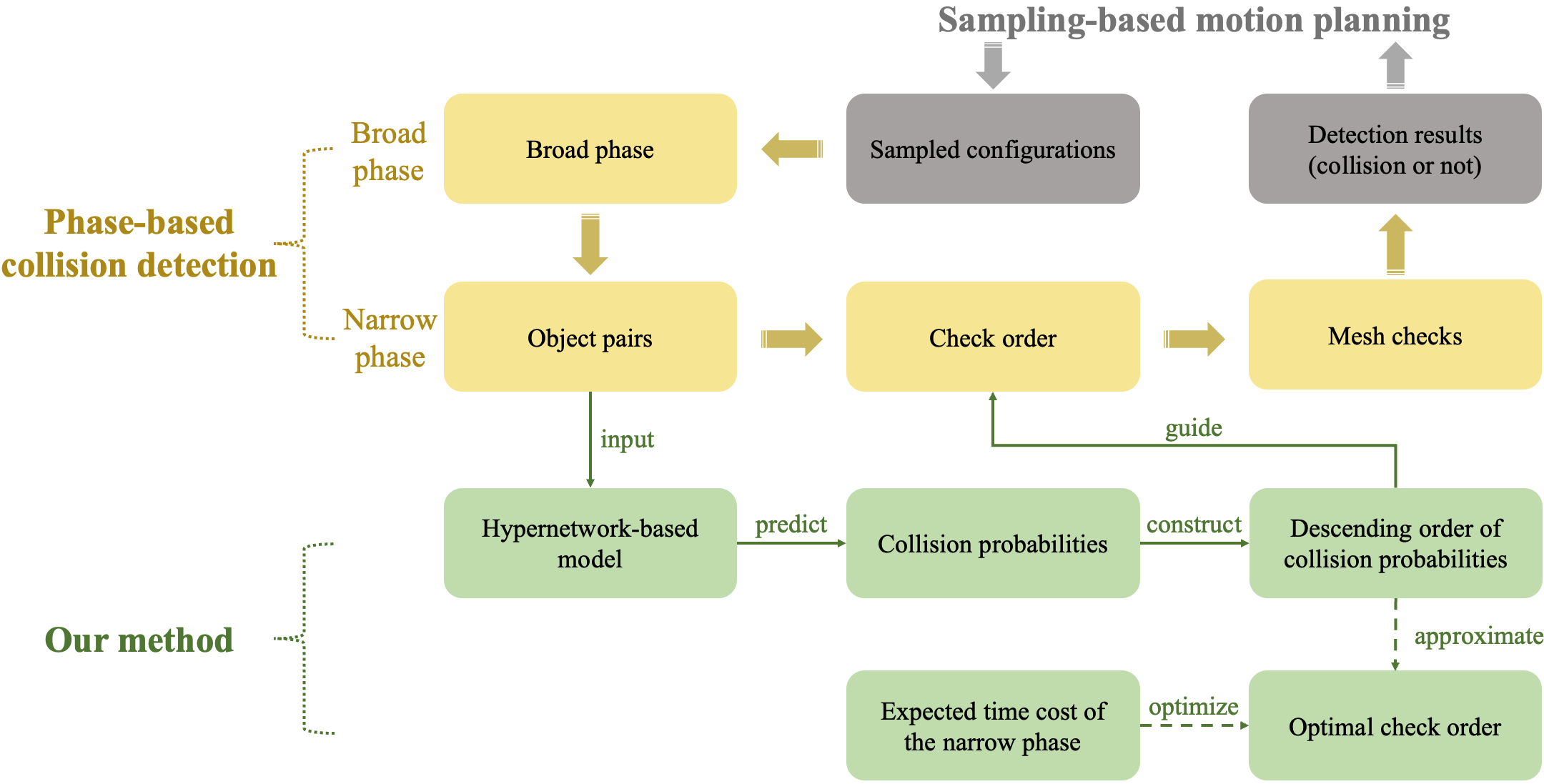}
	\caption{
		The proposed learning-accelerated method for phase-based collision detection in sampling-based motion planning.
	}
	\label{fig:our_method}
\end{figure}

\section{Simulations}
\label{sec:Simulation}
The simulations are built using MoveIt! in Robot Operating System (ROS). The robot model is the Franka Emika Panda, and the obstacle models are selected from ModelNet40~\cite{wu2015ModelNet40}. 
BVH-based collision detection, including both the broad and narrow phases, is implemented using the Flexible Collision Library (FCL)~\cite{2012_ICRA_Pan_fcl} with default parameters in the source code. 
The sampling-based motion planners are implemented using the Open Motion Planning Library (OMPL)~\cite{2012_IRAM_Sucan_ompl}. 
All simulations are conducted on the same computer platform with Ubuntu 20.04, an Intel Core i7-11800H CPU, 32GB RAM, and an NVIDIA RTX 3090 GPU.

For training the hypernetwork-based model, we construct the dataset using $9$ components of the Panda robot model and $252$ obstacle models randomly selected from ModelNet40, forming $9 \times 252$ mesh pairs. For each mesh pair, $100000$ relative poses are randomly generated to create collision-check samples. The ground-truth labels, i.e., collision or non-collision, are obtained by mesh checks using FCL. 
During training, we use the Adaptive Moment Estimation (Adam) optimizer with $1000$  epochs and a learning rate of $10^{-6}$. 
At test time, the mesh information and relative pose, represented by a rotation matrix and a translation vector, are directly available from the simulation environment and used as inputs for network prediction.

\subsection{Performance of Hypernetwork-based Model}

\begin{figure}[!b]
	\centering
	\subfloat[Panda]{\includegraphics[height=1.4in]{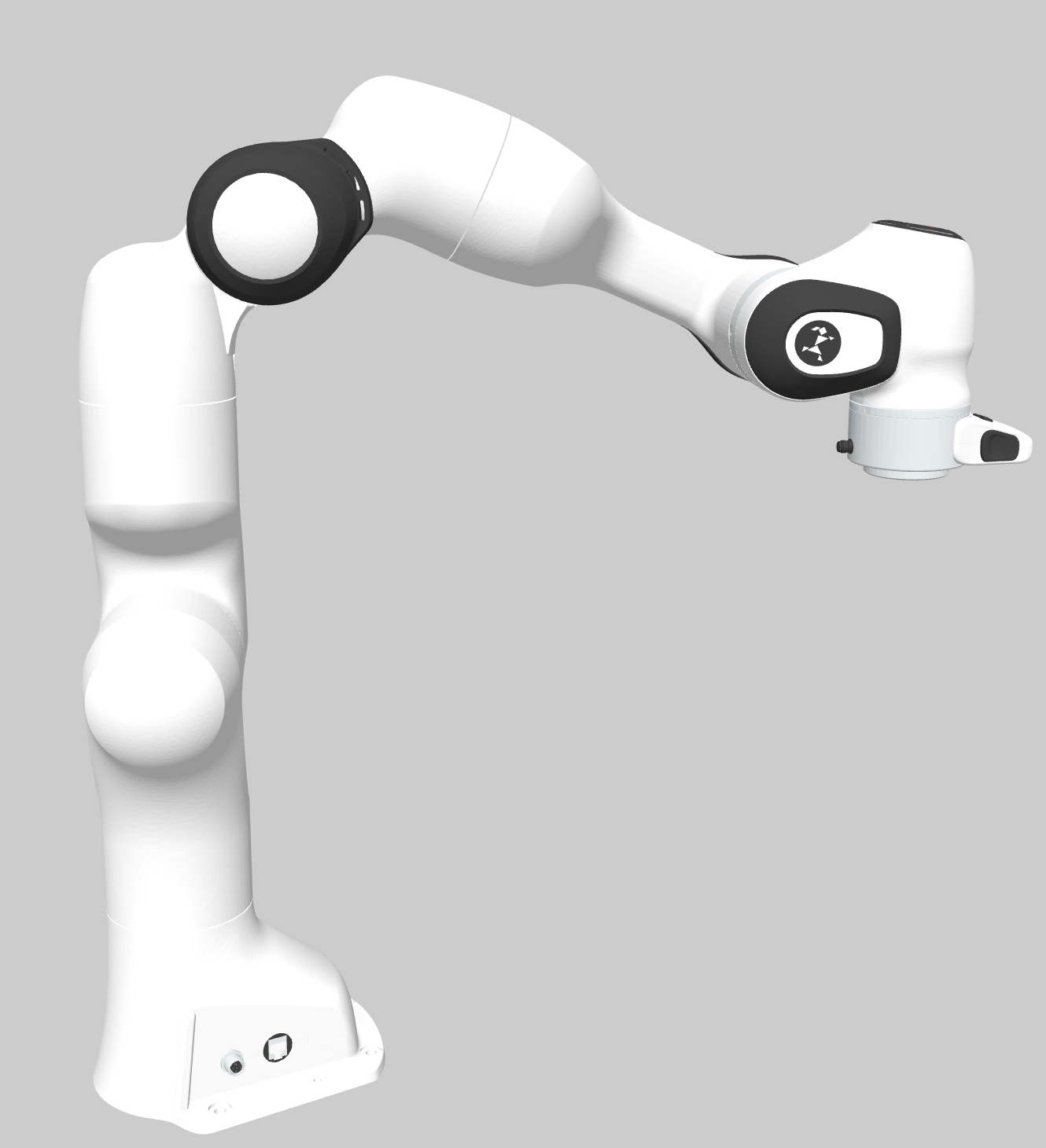}}
	\hspace{0.2in}
	\subfloat[UR5]{\includegraphics[height=1.4in]{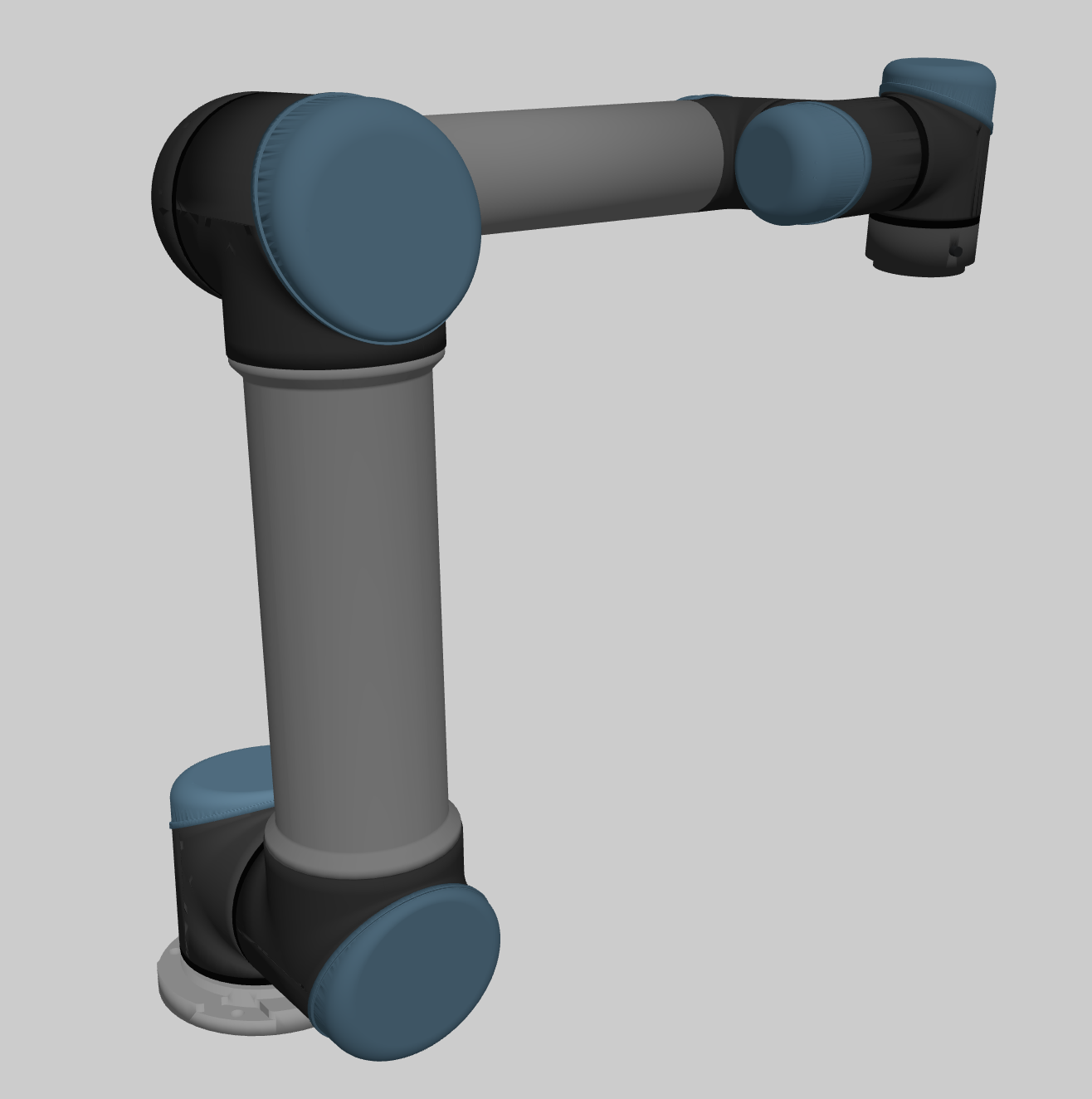}}
	\caption{Simulation models of the robots. 
		The Panda robot is used during training, while the UR5 robot is treated as an unseen robot to further evaluate the generalization capability of the proposed hypernetwork-based model.
	}
	\label{fig:robots}
\end{figure}

\begin{figure}[!b]
	\centering
	\includegraphics[width=3.45in]{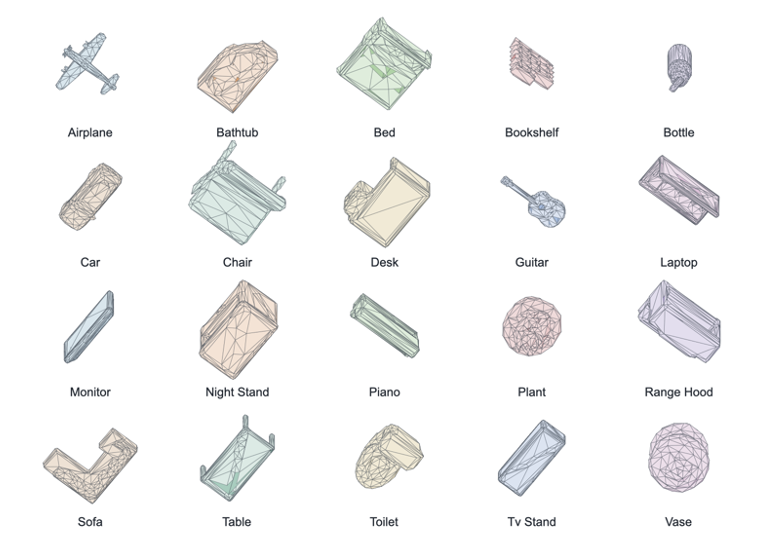}
	\caption{
		Representative meshes from the ModelNet40 dataset.
	}
	\label{fig:modelnet40}
\end{figure}

In this subsection, we evaluate the prediction performance of the hypernetwork-based model, using Fastron~\cite{2020_TRO_Das_fastron} as the baseline. 
Fastron is implemented using the provided code~\cite{2020_TRO_Das_fastron} with default parameters and evaluated on the same testing data as our model.
The testing dataset is divided into two parts. The seen part contains  $2160$ mesh pairs, consisting of $9$ robot components and $240$ obstacles used during training, with $2000$ randomly generated relative poses for each mesh pair. The unseen part contains $432$ mesh pairs, consisting of $9$ robot components and $48$ obstacles  not used during training, with $2000$ randomly generated relative poses for each mesh pair. Each relative pose, together with the corresponding mesh pair, forms a collision-check sample. We compare the proposed model with Fastron in terms of  accuracy, True Positive Rate~(TPR), True Negative Rate~(TNR), and average prediction time. Since the collision label of each sample is binary, the predicted collision probability is binarized using a threshold of $0.5$, where values greater than $0.5$ are classified as collision and values no greater than $0.5$ as non-collision.

The comparison results are shown in Table~\ref{tab:tab_network_validation}: compared with Fastron, the proposed hypernetwork-based model achieves higher prediction accuracy on both seen and unseen object pairs, while requiring less prediction time. This result is consistent with our design motivation of improving generalization while maintaining high accuracy and low computational complexity.
Moreover, the hypernetwork-based model achieves a higher TPR than Fastron, with only a slightly lower TNR. This indicates that our model tends to be more  conservative, prioritizing the detection of potential collisions over the exclusion of non-collision cases. In contrast, Fastron adopts a support-vector-based model, which tends to form larger prediction margins around non-collision samples. Overall, the proposed model demonstrates advantages in prediction accuracy, prediction speed, and generalization ability.
In addition, the results also show that learning-based predictions are not always perfectly accurate. However, in our method, the predicted collision probabilities are used only to optimize the check order in the narrow phase, rather than to replace exact mesh checks. Therefore, prediction errors do not directly change the final collision-detection result, but only affect the efficiency of the checking process. This allows our method to benefit from learning-based prediction while preserving the reliability of the underlying geometric collision detector.
It is also worth noting that, although the hypernetwork-based model is designed for object pairs in the narrow phase, its underlying design idea may inspire future research on collision prediction with more complex inputs, such as predicting collisions involving multiple objects simultaneously.

\begin{figure*}[!b]
	\centering   
	\includegraphics[scale=0.32]{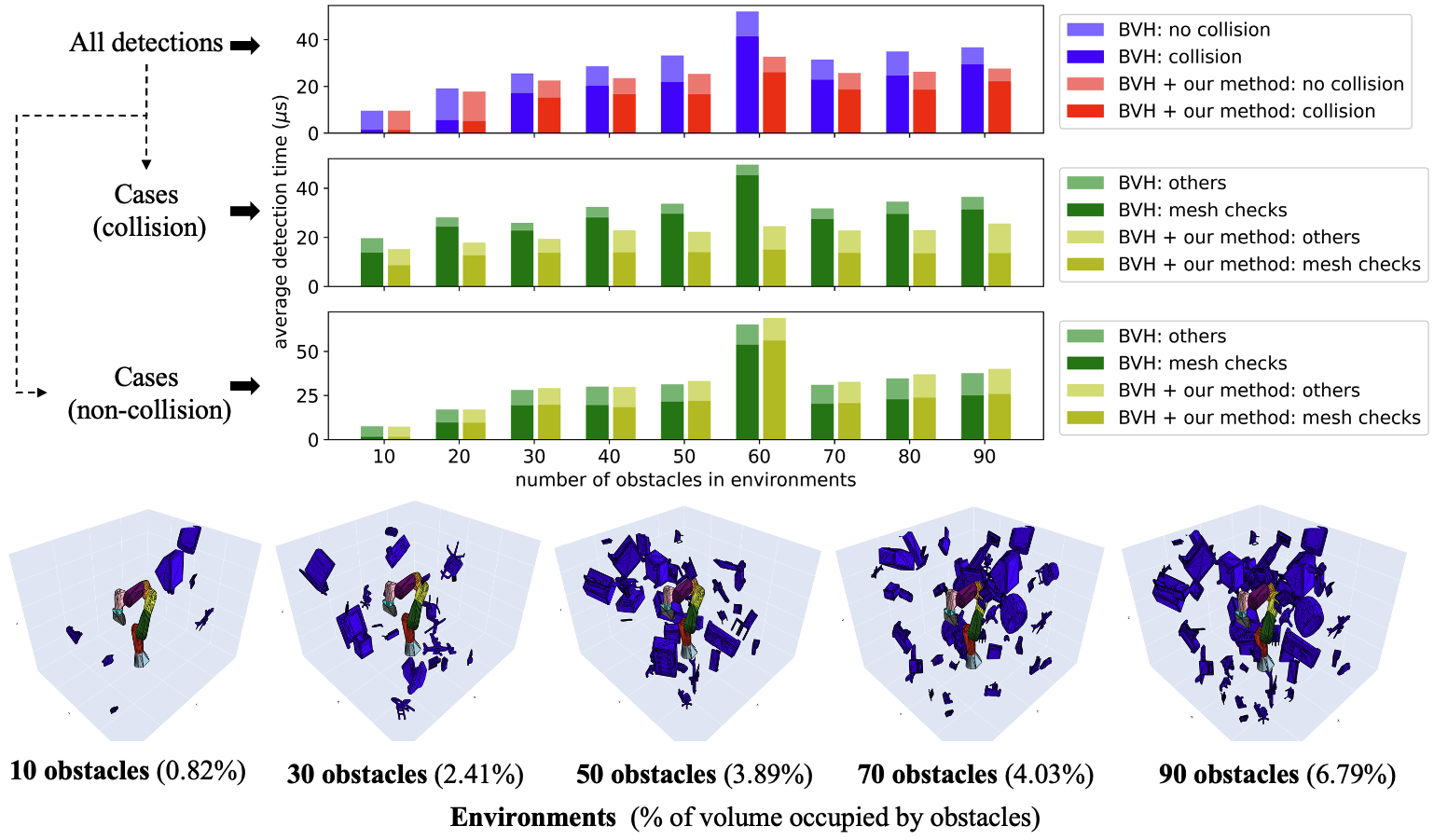}
	\caption{
		Average detection time and its decomposition for all sampled poses in each environment: BVH alone vs. BVH combined with our method.
		The average time is calculated for all detection operations of both approaches, and further separately for the collision and non-collision cases.
		Within the decomposition for collision and non-collision cases: ``mesh checks'' refers to the time cost of executing the narrow phase; 
		for BVH-only, ``others'' corresponds to the time cost of broad phase; 
		for BVH combined with our method, ``others'' includes the time costs of the broad phase,  prediction based on the hypernetwork-based model, and check-order construction based on the predictions.
	}
	\label{fig:obstacles}
\end{figure*}

\begin{table}[!t]
	\scriptsize
	\centering
	\caption{\textbf{[Prediction Comparison on Seen and Unseen Object Pairs]
			\\ The prediction performance of each predictor with Panda}}
	\label{tab:tab_network_validation}
	\centering
	\begin{tabular}{cccccc}
		\toprule
		\textbf{Predictor} &\textbf{Obstacles} & \textbf{Accuracy}&\textbf{TPR}&\textbf{TNR}&\textbf{Time}\\
		\midrule
		\multirow{2}*{Fastron} & seen     & 0.866 & 0.816  & 0.917  &\multirow{2}*{$5.693\,\mu\mathrm{s}$}\\
		& unseen & 0.723 & 0.802 & 0.644  &\\
		\midrule
		\multirow{2}*{ \makecell{Hypernetwork-based \\ model}} & seen     & 0.938 & 0.972  &  0.880 &\multirow{2}*{$1.966\,\mu\mathrm{s}$}\\
		& unseen & 0.926 & 0.965  &  0.861 &\\
		\bottomrule
	\end{tabular}
\end{table}

To evaluate whether the proposed model can generalize to other robot arms without additional training, we test it on an unseen robotic arm, the Universal Robots UR5. Specifically, we use the same obstacle set as that in Table~\ref{tab:tab_network_validation}, but replace the Panda links with the seven links of the UR5 to construct new robot-obstacle mesh pairs. For each mesh pair, we randomly sample $2000$ relative poses to generate collision-check samples for prediction. 
As shown in Table~\ref{tab:new_robot_prediction}, the model effectively predicts collisions for UR5 with both seen and unseen obstacles, demonstrating its generalization capability across environments and robot arms. 

\begin{table}[!t]
	\scriptsize
	\centering
	\caption{\textbf{[Prediction of the UR5 without Additional Training]
			\\ prediction accuracy and runtime in each dataset part }}
	\label{tab:new_robot_prediction}
	\centering
	\begin{tabular}{cccccc}
		\toprule
			\textbf{Robot} &\textbf{Obstacles} & \textbf{Accuracy}&\textbf{TPR}&\textbf{TNR}&\textbf{Time}\\
		\midrule
	\multirow{2}*{UR5 (unseen)}  &  seen & 0.893 & 0.912  & 0.851  & $1.977\,\mu\mathrm{s}$\\
& unseen & 0.878 & 0.907  & 0.823 & $1.993\,\mu\mathrm{s}$\\
		\bottomrule
	\end{tabular}
\end{table}

\subsection{Acceleration on Phase-based Collision Detection}
The proposed method accelerates phase-based collision detection by optimizing the order of exact mesh checks in the narrow phase, without modifying the underlying collision detector. 
This is achieved by using the collision probabilities predicted by the hypernetwork-based model to approximate the optimal check order.
In this subsection, we evaluate whether the proposed method can reduce the time cost of phase-based collision detection.

\begin{table}[!h]
	\scriptsize
	\centering
	\caption{
		\textbf{[Phase-based Collision Detection in the environments]\\Average detection time of Octree-based collision detection  with and without our method}
	}
	\label{tab:octree}
	\centering
	\begin{tabular}{cccc}
		\toprule
		\textbf{Environment} & \textbf{Method} & \textbf{Time} & \makecell[c]{\textbf{Acceleration ratio}\\vs. Octree} \\
		\midrule
		\multirow{2}*{50 obstacles} & Octree & $33.19\,\mu\mathrm{s}$ & -\\
		& Octree + our method & $25.23\,\mu\mathrm{s}$ & $23.98\%$ \\
		\midrule
		\multirow{2}*{70 obstacles} & Octree & $52.01\,\mu\mathrm{s}$ & -\\
		& Octree + our method & $36.68\,\mu\mathrm{s}$ & $29.47\%$ \\
		\bottomrule
	\end{tabular}
\end{table}

First, we construct $9$ simulated environments with different clutter levels, where the number of obstacles ranges from $10$ to $90$. In each environment, the obstacle positions are  randomly generated during initialization and then kept fixed, and we randomly sample $100000$ poses of the Panda robot for collision detection.
For each pose, we perform two collision detection operations: one using BVH alone and the other using BVH combined with our method.
The average detection time and its decomposition in each environment are shown in Fig.~\ref{fig:obstacles}. We have the following observations: 
(\romannumeral1) over all detections, BVH combined  with our method reduces the average detection time compared with BVH alone. The acceleration effect is generally more evident in cluttered environments, achieving a $37\%$ speedup in the scenario with $60$ obstacles. Moreover, as environmental clutteredness increases, a higher proportion of detection time is spent on collision cases. 
(\romannumeral2) in collision cases,  BVH combined with our method incurs  additional costs for probability prediction and  check-order construction, but it reduces the time spent on mesh checks by examining more likely colliding pairs earlier and enabling earlier termination. 
(\romannumeral3) in non-collision cases, all object pairs in the narrow phase must be checked. Therefore,  check-order optimization does not reduce mesh-check time and may slightly increase the total detection time. 
Overall, although our method introduces additional operations, it reduces the narrow-phase mesh-check  time in collision cases by adjusting the check order, thereby accelerating phase-based collision detection. The improvement is substantial in cluttered environments.
Since the proposed method accelerates collision detection by reordering mesh checks in the narrow phase, it can also be applied to other phase-based methods. As shown in Table~\ref{tab:octree}, our method substantially reduces the detection time of Octree-based collision detection in the environments with $50$ and $70$ obstacles shown in Fig.~\ref{fig:obstacles}.

	The generalization capability of the proposed model across robot arms has been evaluated in Sec.~IV-A. We further validate whether the proposed method can accelerate phase-based collision detection for the unseen UR5 without any additional training. The experiments are conducted in two more complex environments containing $100$ and $120$ obstacles. As shown in Table~\ref{tab:new_robot_acc},  our method effectively reduces the detection time for UR5, indicating that the proposed check-ordering strategy can generalize to other robot arms.
\begin{table}[!h]
	\scriptsize
	\centering
	\caption{\textbf{[Collision Detection Time of UR5]
			\\ Average detection time of UR5 with and without our method}}
	\label{tab:new_robot_acc}
	\centering
	\begin{tabular}{cccc}
		\toprule
		\textbf{Environment} & \textbf{Collision detection} & \textbf{Average time} & \makecell[c]{\textbf{Acceleration ratio}\\vs. BVH}\\
		\midrule
		\multirow{2}*{100 obstacles} & BVH   & $245.86\,\mu\mathrm{s}$&  - \\
		& BVH + our method   & $175.15\,\mu\mathrm{s}$ & $28.76\%$  \\
		\midrule
		\multirow{2}*{120 obstacles} &  BVH  & $257.19\,\mu\mathrm{s}$ &  - \\
		& BVH + our method &  $186.99\,\mu\mathrm{s}$  &  $27.29\%$ \\
		\bottomrule
	\end{tabular}
\end{table}

	Finally, we compare our method with the surface area heuristic (SAH) and the distance-based heuristic (DBH), which use geometric metrics as coarse proxies for collision likelihood  to adjust the check order.
	In the environment with $50$ obstacles shown in Fig.~\ref{fig:obstacles}, we randomly resample $1000$ robot poses  and perform collision detection with the two heuristics and our method. 
	The results are shown in Table~\ref{tab:aabb_comparison2}. Both SAH and DBH can  accelerate BVH-based collision detection. However, because these heuristics are simple but coarse, their resulting check orders show noticeable discrepancies from those produced by our model, leading to inferior acceleration performance.

\begin{table}[!h]
	\scriptsize
	\centering
	\caption{\textbf{[Comparison with the heuristic check-ordering methods] \\Time cost in the environment with $50$ obstacles}
	}
	\label{tab:aabb_comparison2}
	\centering
	\begin{tabular}{cccc}
		\toprule
		\textbf{Method} & \textbf{Time} &  \makecell[c]{\textbf{Acceleration ratio}\\vs. BVH} & \makecell[c]{\textbf{Order similarity}\\Spearman coefficient\\ vs. our method}\\
		\midrule 
		BVH & $33.65\,\mu\mathrm{s}$ & - & -\\
		\midrule
		BVH + DBH & $31.19\,\mu\mathrm{s}$ & 7.31\% & 0.587\\
		\midrule
		BVH + SAH & $29.75\,\mu\mathrm{s}$ & 11.58\% & 0.595\\
		\midrule
		BVH + our method& $25.68\,\mu\mathrm{s}$ & 23.69\% &-\\
		\bottomrule
	\end{tabular}
\end{table}

\subsection{Acceleration for Sampling-based Motion Planning}

\begin{table}[!t]
	\scriptsize
	\centering
	\caption{\textbf{[Performance Comparison in Motion Planning]
			\\Average planning time  and success rate of each planner with and without our method}}
	\label{tab:tab2}
	\centering
	\begin{tabular}{cccc}
		\toprule
		\textbf{Motion planner}&\textbf{Collision detection method}&\textbf{Time cost}&\textbf{Success rate}\\
		\midrule
		\multirow{2}*{RRT} & BVH & 1.703 s  & 78.34\% \\
		& BVH + our method & 1.558 s & 81.10\% \\
		\midrule
		\multirow{2}*{RRT-Connect} & BVH & 0.321 s & 98.51\% \\
		& BVH + our method & 0.251 s & 99.46\% \\
		\midrule
		\multirow{2}*{BIT*} & BVH & 0.737 s & 93.02\% \\
		& BVH + our method & 0.635 s & 96.50\% \\
		\bottomrule
	\end{tabular}
\end{table}
As mentioned earlier, sampling-based motion planning involves numerous collision detection operations, which are typically accelerated by phase-based collision detection methods.
In this subsection, we evaluate whether the proposed collision-detection acceleration method can improve the performance of sampling-based motion planners.
Three planners, RRT~\cite{1998_Ames_LaValle_RRT}, RRT-Connect~\cite{2000_ICRA_Kuffner_RRTConnect}, and BIT*~\cite{2015_ICRA_Gammell_BIT*}, are implemented using OMPL.
For each planner, we compare two collision detection settings: BVH alone and BVH combined with our method.
The simulation environment with $50$ obstacles is the same as that shown in Fig.~\ref{fig:obstacles}.
For each planner, we generate $100$ random start-goal configuration pairs in the environment and set a time budget of $5$ seconds for finding a feasible solution. 
Each planner is run $50$ times for each configuration pair and stops when either a feasible solution is found or the time budget is exhausted.

The execution results are shown in Table~\ref{tab:tab2}. For all three planners, combining BVH with our method reduces the average time cost required to find a feasible solution, with RRT-Connect achieving about $22\%$ acceleration. 
Since the proposed method improves the efficiency of phase-based collision detection, the planners can perform more sampling and collision-checking operations within the same  time budget. This enables more thorough exploration of the configuration space, resulting in improved planning success rates for all three planners in the cluttered  environment.

Furthermore, it is worth emphasizing that the planner-level acceleration reported above is achieved under CPU-only inference during planning. GPU support is only required when MFN is invoked to process previously unseen object meshes. After the corresponding CPN parameters are generated, they can be reused for subsequent predictions. This design avoids retraining the entire collision predictor when new object meshes appear.

\section{Conclusion}
\label{sec:conclusion}
In this paper, we propose a learning-based method to accelerate phase-based collision detection in sampling-based motion planning by optimizing the check order in the narrow phase.
First, we formulate the expected time cost of the narrow phase and derive an optimal check-ordering criterion that minimizes this expectation. Second, we design a hypernetwork-based model to learn and predict the priors  required to construct an approximate optimal check order, while maintaining low computational complexity and good generalization capability. 
Simulation results demonstrate that the proposed method substantially reduces the time cost of phase-based collision detection in cluttered environments, and further improves the efficiency and success rate of sampling-based motion planners.

In practical applications, the proposed method can be seamlessly integrated into existing phase-based collision detection frameworks, since it only adjusts the check order in the narrow phase without modifying the underlying collision detector. This non-intrusive design enables straightforward deployment in a wide range of sampling-based motion planning systems with minimal implementation effort. Furthermore, since collision detection often dominates the computational cost of motion planning in cluttered environments, the proposed acceleration strategy has the potential to provide substantial performance gains in real-world robotic applications that require frequent collision checking, such as manipulation, navigation, and motion planning in confined spaces.


\bibliographystyle{IEEEtran}
\bibliography{thesis.bib}
\end{document}